\documentclass{article}

\usepackage{PRIMEarxiv}

\usepackage[utf8]{inputenc} 
\usepackage[T1]{fontenc}    
\usepackage{hyperref}       
\usepackage{url}            
\usepackage{booktabs}       
\usepackage{amsfonts}       
\usepackage{nicefrac}       
\usepackage{microtype}      
\usepackage{lipsum}
\usepackage{fancyhdr}       
\usepackage{graphicx}       
\graphicspath{{media/}}     

\usepackage{verbatim}
\usepackage{amsmath}
\usepackage{amsthm}
\usepackage{booktabs}

\theoremstyle{definition}
\newtheorem{definition}{Definition}[section]
\newtheorem{theorem}{Theorem}
\newtheorem{lemma}{Lemma}

\newcommand{\extremaVgl}{magnitude-comparable} 
\newcommand{\minVgl}{unbiased-trustworthy}
\newcommand{\wordbias}{individual bias}
\newcommand{\setbias}{aggregated bias}
\newcommand{\weatw}{$\textit{WEAT}_{\textit{sample}}$} 
\newcommand{\weat}{$\textit{WEAT}$} 
\newcommand{\db}{\textit{DirectBias}}

\title{Semantic Properties of Cosine Based Bias Scores for Word Embeddings
}

\author{
  Sarah Schr\"oder, Alexander Schulz, Fabian Hinder, Barbara Hammer \\
  Machine Learning Group \\
  Bielefeld University \\
  Bielefeld, Germany\\
  \texttt{\{saschroeder, aschulz, fhinder, bhammer\}@techfak.uni-bielefeld.de} \\
}

\begin{document}
\maketitle

\begin{abstract}
Plenty of works have brought social biases in language models to attention and proposed methods to detect such biases. As a result, the literature contains a great deal of different bias tests and scores, each introduced with the premise to uncover yet more biases that other scores fail to detect. What severely lacks in the literature, however, are comparative studies that analyse such bias scores and help researchers to understand the benefits or limitations of the existing methods. 
In this work, we aim to close this gap for cosine based bias scores. By building on a geometric definition of bias, we propose requirements for bias scores to be considered meaningful for quantifying biases. Furthermore, we formally analyze cosine based scores from the literature with regard to these requirements. We underline these findings with experiments to show that the bias scores' limitations have an impact in the application case.

This is a preprint of the conference paper \cite{icpram24schroeder}. An earlier draft can be found on arXiv \cite{schröder2021evaluatingmetricsbiasword}.
\end{abstract}

\section{Introduction}
\label{sec:intro}
In the domain of Natural Language Processing (NLP), many works have investigated social biases in terms of associations in the embeddings space. Early works \cite{bolukbasi,weat} introduced methods to measure and mitigate social biases based on cosine similarity in word embeddigs. With NLP research progressing to large language models and contextualized embeddings, doubts have been raised whether these methods are still suitable for fairness evaluation \cite{seat} and other works criticize that for instance the Word Embedding Association Test (WEAT) \cite{weat} fails to detect some kinds of biases \cite{lipstick,ripa}. Overall there exists a great deal of bias measures in the literature, which not necessarily detect the same biases \cite{kurita,lipstick,ripa}. In general, researchers are questioning the usability of model intrinsic bias measures, such as cosine based methods \cite{steed2022upstream,goldfarb2020intrinsic,kaneko2022debiasing}. There exist few papers that compare the performance of different bias scores \cite{delobelle2021measuring,schroder2023so} and works that evaluate experimental setups for bias measurement \cite{seshadri2022quantifying}. However, to our knowledge, only two works investigate the properties of intrinsic bias scores on a theoretical level \cite{ripa,du-etal-2021-assessing}.
To further close this gap, we evaluate the semantic properties of cosine based bias scores, focusing on bias quantification as opposed to bias detection. We make the following contributions:
(i) We formalize the properties of trustworthiness and comparability as requirements for cosine based bias scores.
(ii) We analyze WEAT and the Direct Bias, two prominent examples from the literature.
(iii) We conduct experiments to highlight the behavior of WEAT and the Direct Bias in practice.

Both our theoretical analysis and experiments show limitations of these bias scores in terms of bias quantification. It is crucial that researchers take these limitations into account when considering WEAT or the Direct Bias for their works. Furthermore, we lay the ground work to analyze other cosine based bias scores and understand how they can be useful for the fairness literature.
The paper is structured as follows: In Section \ref{sec:related_work} we summarize WEAT, the Direct Bias and general terminology for cosine based bias measures from the literature. We introduce formal requirements for such bias scores in Section \ref{sec:req} and analyze WEAT and the Direct Bias in terms of these requirements in Section \ref{sec:score_analysis}. In Section \ref{sec:experiments} we support our theoretical findings by experiments, before drawing our conclusions in Section \ref{sec:conclusion}.

\section{Related Work for Bias in World Embeddings}
\label{sec:related_work}
\subsection{WEAT}

The Word Embedding Association Test, short WEAT, \cite{weat}, is a statistical test for stereotypes in word embeddings. 
The test compares two sets of target words $X$ and $Y$ with two sets of bias attributes $A$ and $B$ of equal size $n$ under the hypothesis that words in $X$ are rather associated with words in $A$ and words in $Y$ rather associated with words in $B$. The association of a single word $\mathbf{w}$ with the bias attribute sets $A$ and $B$ including $n$ attributes each, is given by
\begin{eqnarray}
\label{eq:weat_attr_sim}
s(\mathbf{w},A,B) = \frac{1}{n} \sum_{\mathbf{a} \in A}\cos(\mathbf{w},\mathbf{a}) - \frac{1}{n} \sum_{\mathbf{b} \in B}\cos(\mathbf{w},\mathbf{b}).
\end{eqnarray}
To measure bias in the sets $X$ and $Y$, the effect size is used, which is a normalized measure for the association difference between the target sets
\begin{eqnarray}
\label{eq:weat_eff_size}
d(X,Y,A,B) = \frac{ \frac{1}{n} \sum_{\mathbf{x} \in X} s(\mathbf{x},A,B) - \frac{1}{n} \sum_{\mathbf{y} \in Y} s(\mathbf{y},A,B)}{stddev_{\mathbf{w} \in X \cup Y} s(\mathbf{w},A,B)}.
\end{eqnarray}
A positive effect size confirms the hypothesis that words in $X$ are rather stereotypical for the attributes in $A$ and words in $Y$ stereotypical for words in $B$, while a negative effect size indicates that the stereotypes would be counter-wise. To determine if the effect is indeed statistically significant, the permutation test
\begin{eqnarray}
\label{eq:weat_perm}
p &=& P_r [ s(X_i,Y_i,A,B) > s(X,Y,A,B)].
\end{eqnarray}
with subsets $(X_i,Y_i)$ of $X \cup Y$ and the test statistic
\begin{eqnarray}
\label{eq:weat_test_stat}
s(X,Y,A,B) &=& \sum_{\mathbf{x} \in X} s(\mathbf{x},A,B) - \sum_{\mathbf{y} \in Y} s(\mathbf{y},A,B)
\end{eqnarray}
is done. As a statistical test WEAT is suited to confirm a hypothesis (such that a certain type of stereotype exists in a model), but it cannot prove the opposite.

\subsection{Direct Bias}
\label{sec:Bolukbasi-directbias}

The Direct Bias \cite{bolukbasi} is defined as the correlation of neutral words $\mathbf{w} \in W$ with a bias direction (for example gender direction $\mathbf{g}$):
\begin{eqnarray}
\label{eq:direct_bias}
DirectBias(W) := \frac{1}{|W|} \sum_{\mathbf{w} \in W} |\cos(\mathbf{w},\mathbf{g})|^c 
\end{eqnarray}

with $c$ determining the strictness of bias measurement. The gender direction is either obtained by a gender word-pair e.g. $\mathbf{g} = \mathbf{he} - \mathbf{she}$ or - to get a more robust estimate - it is obtained by computing the first principal component over a set of individual gender directions from different word-pairs.\\
In terms of their debiasing algorithm the authors describe how to obtain a bias subspace given defining sets $D_1$, ..., $D_n$. A defining set $D_i$ includes words $\mathbf{w}$ that only differ by the bias relevant topic e.g. for gender bias $\{\mathbf{man},\mathbf{woman}\}$ could be used as a defining set. Given these sets, the authors construct individual bias directions $\mathbf{w} - \mathbf{\mu_i} \; \forall \mathbf{w} \in D_i, i \in \{1,...,n\}$ and $\mathbf{\mu_i} = \sum_{\mathbf{w} \in D_i} \frac{\mathbf{w}}{|D_i|}$. To obtain a k-dimensional bias subspace $B$ they compute the $k$ first principal components over these samples.

\subsection{Terminology}
In the literature geometrical bias is measured by comparing neutral targets against sensitive attributes. By targets and attributes we refer to vector representations of words, sentences or text in a $d$-dimensional embedding space. However, the methodology can be applied to any kind of vector representations. While the exact notation varies between publications, we summarize and use it in the following Sections as follows:

Given a protected attribute like gender or race, we select $n \geq 2$ protected groups that might be subject to biases. Each protected group is defined by a set of attributes $\mathbf{a_{ik}} \in A_i$ with $i \in \{1, ..., n\}$ the group's index. We summarize these attribute sets as $A = \{A_1, ..., A_n\}$. The intuition is that the attributes define the relation of protected groups by contrasting specifically over the membership to the different groups. Therefore, it is important that any attribute $\mathbf{a_{ik}} \in A_i$ has a counterpart $\mathbf{a_{jk}} \in A_j \; \forall \; A_j \in A, j \neq i$ that only differs from $\mathbf{a_{ik}}$ by the group membership. For instance, if we used $A_1 = \{she, female, woman\}$ as a selection of female terms, $A_2 = \{he, male, man\}$ would be the proper choice of male terms.

Analogously to WEAT's definition of word biases, we define the association of a target $\mathbf{t}$ with one protected group, represented by $A_i$, as
\begin{align}
\label{eq:group_sim}
s(\mathbf{t},A_i) = \frac{1}{|A_i|} \sum_{\mathbf{a_{ik}} \in A_i} cos(\mathbf{t}. \mathbf{a_{ik}})
\end{align}
A similar notion is found with the Direct Bias \cite{bolukbasi}. To detect bias, one would consider the difference of associations towards the different groups, i.e. is $\mathbf{t}$ more similar to one protected group than the others. This concept is also found in most cosine based bias scores.\\
Whether such association differences are harmful depends on whether $\mathbf{t}$ is theoretically neutral to the protected groups. For example, terms like "aunt" or "uncle" are associated with one or the other gender per definition, while a term like "nurse" should not be associated with gender.

\section{Formal Requirements for Bias Scores}
\label{sec:req}
\subsection{Formal Bias Definition and Notations}
\label{sec:formal_bias_def}
As baseline for our bias score requirements and the following analysis of bias scores from the literature, we suggest two intuitive definitions of individual bias for target samples (e.g. one word) $\mathbf{t}$ and aggregated biases for sets of targets $T$.
For samples $\mathbf{t}$ we apply the intuition of WEAT, extended to $n$ protected groups instead of only two.

\begin{definition}[\wordbias]
	\label{def:word_bias}
	Given $n$ protected groups represented by attribute sets $A_1$, ..., $A_n$ and a target $\mathbf{t}$ that is theoretically neutral to these groups, we consider $\mathbf{t}$ biased if 
	\begin{align}
	\label{eq:word_bias}
    \exists A_i, A_j \in A: s(\mathbf{t},A_i) > s(\mathbf{t},A_j)
	\end{align}
\end{definition}

\begin{definition}[\setbias] 
	\label{def:set_bias}
	Given $n$ protected groups represented by attribute sets $A_1$, ..., $A_n$ and a set of targets $T$ containing only samples that are theoretically neutral to these groups, we consider $T$ biased if at least one sample $\mathbf{t} \in T$ is biased:
	\begin{align}
	\label{eq:set_bias}
    \exists A_i, A_j \in A, \mathbf{t} \in T: s(\mathbf{t},A_i) > s(\mathbf{t},A_j)
	\end{align}
\end{definition}

The idea behind Definition \ref{def:set_bias} is that even when looking at aggregated biases, each individual bias is important, i.e. as long as there is one biased target in the set, we cannot call the set unbiased, even if target biases cancel out on average or the majority of targets is unbiased.

In the following we will use a notation for bias score functions in general: $b(\mathbf{t},A)$ measuring the bias of one target and $b(T,A)$ for aggregated biases. Note that there are two different strategies in the literature: Bias scores measuring bias over all neutral words jointly (Direct Bias), which matches our notation $b(T,A)$, and bias scores measuring the bias over two groups of neutral words $X,Y \subset T$ (WEAT). In the later case, we consider the selection of subsets $X,Y \subset T$ as part of the bias score and thus treat it as a function $b(T,A)$.

Since the bias scores from the literature have different extreme values and different values indicating no bias, we use the following notations: $b_{min}$ and $b_{max}$ are the extreme values of $b(\cdot)$ and $b_0$ is the value of $b(\cdot)$ that means $\mathbf{t}$ or $T$ is unbiased. Note that $b_{min}$ and $b_0$ are not necessarily equal.

\subsection{Requirements for Bias Metrics}
\label{sec:requirements}

Based on the definitions of bias explained in Section \ref{sec:formal_bias_def}, we formalize the properties of \emph{trustworthiness} and \emph{magnitude-comparability}. 
The goal of both properties is to ensure that biases can be quantified in a way such that bias scores can be safely compared between different embedding models and debiasing methods can be evaluated without risking to overlook bias.

\subsubsection{Comparability}
The goal of \emph{magnitude-comparability}, is to ensure that bias scores are comparable between embeddings of different models. This is necessary to make statements about embedding models being more or less biased than others, which includes comparing debiased embeddings with their original counterparts.
We find a necessary condition for such comparability is the possibility to reach the extreme values $b_{min}$ and $b_{max}$ of $b(\cdot)$ in different embedding spaces depending only on the neutral targets and their relation to attribute vectors, as opposed to the attribute vectors themselves, which might be embedded differently given different models.

\begin{definition}[\extremaVgl]
	\label{def:max_amplitutde}
	We call the bias score function $b(T,A)$ \extremaVgl\ if, for a fixed number of target samples in set $T$ (including the case $T = \{\mathbf{t}\}$), the maximum bias score $b_{max}$ and the minimum bias score $b_{min}$ are independent of the attribute sets in A:
	\begin{align}
	\label{eq:max_min_value}
	\max_{T, |T|=const} b(T,A) = b_{max} \; \forall \;  A, \quad  \\
	\min_{T, |T|=const} b(T,A) = b_{min} \; \forall \;  A.
	\end{align}
\end{definition}

\subsubsection{Trustworthiness}
The second property of \emph{trustworthiness} defines whether we can trust a bias score to report any bias in accordance to Definitions \ref{def:word_bias} and \ref{def:set_bias}, i.e. the bias score can only reach $b_0$, which indicates fairness, if the observed target is equidistant to all protected groups and for target sets if all samples in the observed set of targets are unbiased. This is important, because even if a set of targets is mostly unbiased or target biases cancel out on average, individual biases can still be harmful and should thus be detected.
The requirement for the consistency of the minimal bias score $b_0$ can be formulated in a straight forward way using the similarities to the attribute sets $A_i$.

\begin{definition}[\minVgl]
	\label{def:min_max_bias}
	Let $b_0$ be the bias score of a bias score function, that is equivalent to no bias being measured.
	We call the bias score function $b(\mathbf{t},A)$ 
	\minVgl\ if 
	\begin{align}
	\label{eq:cond_neutral}
	b(\mathbf{t},A) = b_0 \iff s(\mathbf{t},A_i) = s(\mathbf{t},A_j) \; \forall \;  A_i, A_j \in A.
	\end{align}
	Analogously for aggregated scores with a set $T = \{ \mathbf t_1, ..., \mathbf t_m\}$,
	we say $b(T,A)$ is \minVgl\ if
	\begin{align}
	\label{eq:cond_neutral2}
	b(T,A) &= b_0 \\
    \iff s(\mathbf{t}_k,A_i) &= s(\mathbf{t}_k,A_j) \; \forall \;  A_i, A_j \in A, k\in\{1,...,m\}.
	\end{align}
\end{definition}

\section{Analysis of Bias Scores}
\label{sec:score_analysis}
As a major contribution of this work, we formally analyze WEAT and the Direct Bias with regard to the properties defined in Section \ref{sec:requirements}. Table \ref{tab:overview} gives an overview over the properties. The detailed analyses follow in Section \ref{sec:weat_analysis} for WEAT and Section \ref{sec:analysis_db} for the Direct Bias.

\begin{table}[b]
\centering
\caption{Overview over the properties of bias scores.}
\label{tab:overview}
\begin{tabular}{ ccc }
\toprule
bias score & comparable & trustworthy \\
\midrule
\weatw &  x & \checkmark \\
\weat &  \checkmark & x  \\
\db &  \checkmark & x \\
\bottomrule
\end{tabular}
\end{table}


\subsection{Analysis of WEAT}
\label{sec:weat_analysis}
In the following, we detail properties of WEAT in light of the definitions stated above. First, we focus on the individual biases as reported by $s(\mathbf{t},A,B)$.

\begin{theorem}
\label{theor:weat_s_extrema}
	The bias score function $s(\mathbf{t},A,B)$ of WEAT is not \extremaVgl.
\end{theorem}
\begin{proof}
\label{proof:weat_s_extrema}
With $\mathbf{\hat{a}} = \frac{1}{|A|} \sum_{\mathbf{a} \in A} \frac{\mathbf{a}}{||\mathbf{a}||}$ and $\mathbf{\hat{b}}$ analogously defined, we can rewrite
\begin{align}
    s(\mathbf{t},A,B) &=& \frac{\mathbf{t} \cdot \mathbf{\hat{a}}}{||\mathbf{t}||} - \frac{\mathbf{t} \cdot \mathbf{\hat{b}}}{||\mathbf{t}||} \\
    &=& \frac{\mathbf{t}}{||\mathbf{t}||} \cdot \Big(\mathbf{\hat{a}} - \mathbf{\hat{b}} \Big) \\
    &=& cos(\mathbf{t},\mathbf{\hat{a}} - \mathbf{\hat{b}}) ||\mathbf{\hat{a}} - \mathbf{\hat{b}}||.
\end{align}
Hence we can show that the extreme values depend on the attribute sets $A$ and $B$:
\begin{align}
    max_{\mathbf{t}} s(\mathbf{t},A,B) = ||\mathbf{\hat{a}} - \mathbf{\hat{b}}||, \\
    min_{\mathbf{t}} s(\mathbf{t},A,B) = -||\mathbf{\hat{a}} - \mathbf{\hat{b}}||
\end{align}
The statement follows.
\end{proof}

\begin{theorem}
\label{theor:weat_s_min}
	The bias score function $s(\mathbf{t},A,B)$ of WEAT is \minVgl.
\end{theorem}
\begin{proof}
\label{proof:weat_s_min}
This follows directly from the definition of $s(\mathbf{t},A,B)$ (equation \eqref{eq:weat_attr_sim}):
\begin{align}
& s(\mathbf{t},A,B) = s(\mathbf{t},A) - s(\mathbf{t},B) = 0 \\
\iff & s(\mathbf{t},A) = s(\mathbf{t},B)
\end{align}
\end{proof}

Next, we focus on the properties of the effect size $d(X,Y,A,B)$, identified by \weat\ in Table \ref{tab:overview}. Note that it is not specified for cases, where  $s(\mathbf{t},A,B) = s(\mathbf{t'},A, B) \;  \forall \mathbf{t}, \mathbf{t'} \in X \cup Y$ due to its denominator. This is highly problematic considering Definition \ref{def:min_max_bias}, which states that a bias score should be $0$ in that specific case. For Theorem \ref{theor:weat_extrema} we need Lemma \ref{theor:schrankeAllg} from the Appendix.

\begin{theorem}
\label{theor:weat_d_min}
	The effect size $d(X,Y,A,B)$ of WEAT is not \minVgl.
\end{theorem}
\begin{proof}
\label{proof:weat_d_min}
For the WEAT score $b_0 = 0$. 
With four targets $\mathbf{t}_1, \mathbf{t}_2, \mathbf{t}_3, \mathbf{t}_4$ and $s(\mathbf{t}_1,A,B) = s(\mathbf{t}_3,A,B)$ and $s(\mathbf{t}_2,A,B) = s(\mathbf{t}_4,A,B)$ the effect size
\begin{align}
\label{eq:eff_size}
& d(\{\mathbf{t}_1,\mathbf{t}_2\},\{\mathbf{t}_3,\mathbf{t}_4\},A,B) = \\ 
& \frac{(s(\mathbf{t}_1,A,B) + s(\mathbf{t}_2,A,B)) - (s(\mathbf{t}_3,A,B) + s(\mathbf{t}_4,A,B))}{2 \cdot \mathrm{stddev}_{\mathbf{t} \in \{\mathbf{t}_1,\mathbf{t}_2,\mathbf{t}_3, \mathbf{t}_4\}} s(\mathbf{t},A,B)}
\end{align}
is $0$, if $s(\mathbf{t}_1,A,B) \neq s(\mathbf{t}_2,A,B)$ (otherwise $d$ is not defined). 
Now, for the simple case $A=\{\mathbf{a}\}, B=\{\mathbf{b}\}$ and assuming all vectors having length $1$, we see
\begin{align}
&s(\mathbf{t}_1,A,B) = s(\mathbf{t}_3,A,B) \nonumber \\
\iff & \mathbf{a}\cdot\mathbf{t}_1 - \mathbf{b}\cdot\mathbf{t}_1 = \mathbf{a}\cdot\mathbf{t}_3 - \mathbf{b}\cdot\mathbf{t}_3 \nonumber \\
\iff & \mathbf{a}\cdot(\mathbf{t}_1 - \mathbf{t}_3) - \mathbf{b}\cdot(\mathbf{t}_1 - \mathbf{t}_3) = 0 \nonumber \\
\iff & (\mathbf{a} - \mathbf{b})\cdot(\mathbf{t}_1 - \mathbf{t}_3) = 0.
\end{align}
This implies that, if the two vectors $\mathbf{a} - \mathbf{b}$ and $\mathbf{t}_1 - \mathbf{t}_3$ are orthogonal (and e.g.\ $s(\mathbf{t}_2,A,B)=0$), the WEAT score returns 0. In this case, there exist $\mathbf{a}, \mathbf{b}, \mathbf{t}_1, \mathbf{t}_3$ with $s(\mathbf{t}_1,A,B) = s(\mathbf{t}_3,A,B) \neq 0$ and accordingly $s(\mathbf{t}_1,A) \neq s(\mathbf{t}_1,B)$.
\end{proof}

\begin{theorem}
\label{theor:weat_extrema}
The effect size $d(X,Y,A,B)$ of WEAT with $X=\{\mathbf{x}_1, \ldots,\mathbf{x}_m\}, Y=\{\mathbf{y}_1, \ldots,\mathbf{y}_m\}$ is \extremaVgl.
\end{theorem}
\begin{proof}
\label{proof:weat_extrema}
With $c_i = s(\mathbf{x_i},A,B)$, $c_{i+m} = s(\mathbf{y_i},A,B)$, $n=2m$, $\hat{\mu}=1/n \sum^{n}_{i=1} c_i$ and $\hat{\sigma}=\sqrt{1/n\sum^{n}_{i=1} (c_i - \mu)^2}$, we have
\begin{align}
    d &= \frac{1/m \sum^m_{i=1} c_i - 1/m \sum^{2m}_{i=m+1} c_i}{\hat{\sigma}} \\
      &= \frac{\sum^m_{i=1} c_i - \sum^{2m}_{i=m+1} c_i + \sum^m_{i=1} c_i - \sum^m_{i=1} c_i}{m\hat{\sigma}} \nonumber \\
      &= \frac{2\sum^m_{i=1} c_i - 2m\hat{\mu}}{m\hat{\sigma}} \nonumber \\
      &= \frac{2}{m} \sum^m_{i=1}\frac{c_i - \hat{\mu}}{\hat{\sigma}} \in[-2,2]
\end{align}
where the last statement follows from Lemma \ref{theor:schrankeAllg} (see Appendix) with $\sum^m_{i=1}\frac{c_i - \hat{\mu}}{\hat{\sigma}} \in [-m,m]$. The extreme value $\pm 2$ is reached if $c_1=\ldots =c_m=-c_{m+1}=\ldots =-c_{2m}$, which can be obtained by setting $\mathbf{x}_1=\ldots =\mathbf{x}_m=-\mathbf{y}_1=\ldots =-\mathbf{y}_m$, independently of $A$ and $B$ as long as $A \neq B$ and $\sum_{\mathbf{a}_i \in A} \mathbf{a}_i/\|\mathbf{a}_i\| \neq 0 \neq \sum_{\mathbf{b}_i \in B} \mathbf{b}_i/\|\mathbf{b}_i\|$.
\end{proof}

The proof of Theorem \ref{theor:weat_extrema} shows that the effect size reaches its extreme values only if all $x \in X$ achieve the same similarity score $s(\mathbf{x},A,B)$ and $s(\mathbf{y},A,B) = -s(\mathbf{x},A,B) \; \forall \; y \in Y$, i.e. the smaller the variance of $s(\mathbf{x},A,B)$ and $s(\mathbf{y},A,B)$ the higher the effect size. This implies that we can influence the effect size by changing the variance of $s(\mathbf{t},A,B)$ without changing whether the groups are separable in the embedding space. Furthermore, the proof of Theorem \ref{theor:weat_d_min} shows that WEAT can report no bias even if the embeddings contain associations with the bias attributes. This problem occurs, because WEAT is only sensitive to the stereotype "$X$ is associated with $A$ and $Y$ is associated with $B$" and will overlook biases diverting from this hypothesis.

\subsection{Analysis of the Direct Bias}
\label{sec:analysis_db}

For the Direct Bias, the following theorems show that it is \extremaVgl, but not \minVgl.
The proof for Theorem \ref{theor:direct_bias_minvgl} shows that the first principal component used by the Direct Bias does not necessarily represent individual bias directions appropriately. This can lead to both over- and underestimation of bias by the Direct Bias.

\begin{theorem}
\label{theor:direct_bias_extrema}
The \db\ is \extremaVgl\ for $c \geq 0$.
\end{theorem}
\begin{proof}
\label{proof:direct_bias_extrema}
For $c \geq 0$ the individual bias $|\cos(\mathbf{t},\mathbf{g})|^c$ is in $[0,1]$. Calculating the mean over all targets in $T$ does not change this bound. The statement follows.
\end{proof}

\begin{theorem}
\label{theor:direct_bias_minvgl}
The \db\ is not \minVgl.
\end{theorem}
\begin{proof}
\label{proof:db_min}
For the Direct Bias $b_0 = 0$ indicates no bias.
Consider a setup with two attribute sets $A_=\{\mathbf{a_1}, \mathbf{a_2}\}$ and $C=\{\mathbf{c_1}, \mathbf{c_2}\}$.\\
Using the notation from Section \ref{sec:Bolukbasi-directbias} this gives us two defining sets $D_1 = \{\mathbf{a_1}, \mathbf{c_1}\}$ $D_2 = \{\mathbf{a_2}, \mathbf{c_2}\}$. Let $\mathbf{a_{1}} = (-x, rx)^T = -\mathbf{c_{1}}, \mathbf{a_{2}}  = (-x, -rx)^T = -\mathbf{c_{2}}$ and $r > 1$.\\
The bias direction is obtained by computing the first principal component over all $(\mathbf{a_{i}} - \mathbf \mu_i)$ and $(\mathbf{c_{i}} - \mathbf \mu_i)$ with $\mathbf \mu_i = \frac{\mathbf{a_i}+\mathbf{c_i}}{2} = 0$. Due to $r > 1$, $\mathbf{b} = (0, 1)^T$ is a valid solution for the 1st principal component as it maximizes the variance
\begin{align}
    \mathbf{b} = argmax_{\|\mathbf{v}\| = 1} \sum_i (\mathbf{v} \cdot \mathbf{a_{i}})^2 + (\mathbf{v} \cdot \mathbf{c_{i}})^2.
\end{align}

According to the definition in Section \ref{sec:formal_bias_def}, any word $\mathbf{t} = (0, w_y)^T$ would be considered neutral to groups $A$ and $C$ with $s(\mathbf{t}, A) = s(\mathbf{t}, C)$ and being equidistant to each word pair $\{a_i, c_i\}$.\\
But with the bias direction $\mathbf{b} = (0, 1)^T$ the Direct Bias would report $b_{max} = 1$ instead of $b_0 = 0$, which contradicts Definition \ref{def:min_max_bias}.\\
On the other hand, we would consider a word  $\mathbf{t} = (w_x, 0)^T$ maximally biased, but the Direct Bias would report no bias.
Showing that the bias reported for single words $\mathbf{t}$ is not \minVgl, proves that the \db\ is not \minVgl.
\end{proof}

\section{Experiments}
\label{sec:experiments}
In the experiments, we show that the limitations of WEAT and Direct Bias shown in Section \ref{sec:score_analysis} do occur with state-of-the-art language models. We show that the effect size of WEAT can be misleading when comparing bias in different settings. Furthermore, we highlight how attribute embeddings differ between different models, which impacts WEAT's individual bias, and that the Direct Bias can obtain a misleading bias direction by using the Principal Component Analysis (PCA). We use different pretrained language models from Huggingface \cite{huggingface} and the PCA implementation from Scikit-learn \cite{scikit} and observe gender bias based on 25 attributes per gender, such as $(\mathbf{man}, \mathbf{woman})$.

\begin{figure*}[b]
\centering
\includegraphics[width=0.45\linewidth]{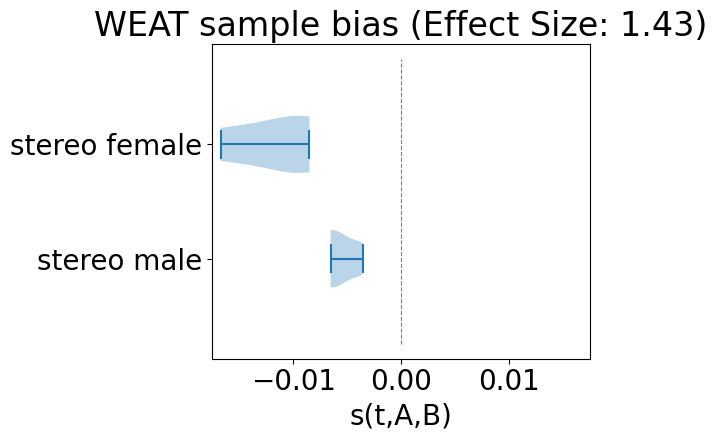}
\includegraphics[width=0.45\linewidth]{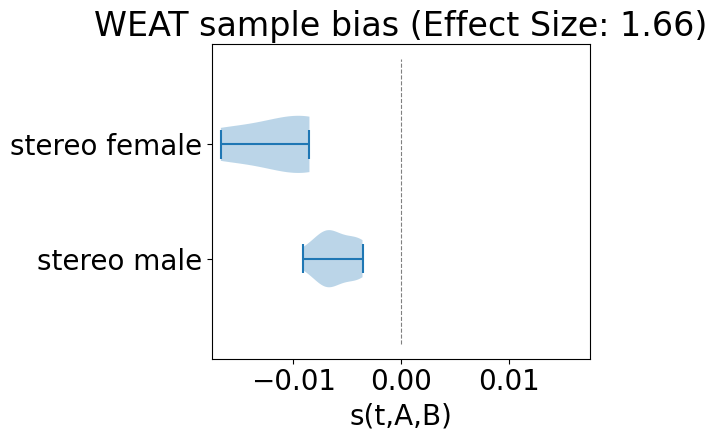}
\caption{WEAT individual bias and effect size for distilBERT with different selections of target words. When selecting a smaller number of job titles (left), we observe that stereotypical male/female jobs are more distinct w.r.t. $s(\mathbf{t},A,B)$, while the effect size is lower.}
\label{fig:eff_size}
\end{figure*}

\subsection{WEAT's effect size}

\begin{figure}[t]
\centering
\includegraphics[width=0.5\linewidth]{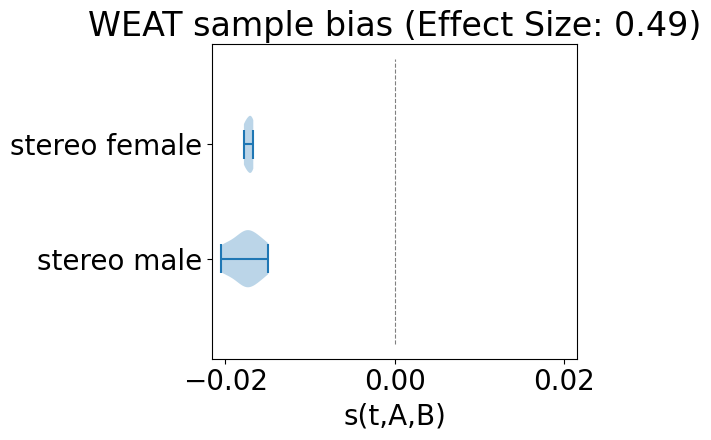}
\caption{WEAT's individual bias for job titles in GPT.}
\label{fig:gpt_weat}
\end{figure}

In a first experiment, we demonstrate that the effect size does not quantify social bias in terms of the separability of stereotypical targets. We use embeddings of \textit{distilbert-base-uncased} and \textit{openai-gpt} to compute gender bias according to $s(\mathbf{t},A,B)$ and the effect size $d(X,Y,A,B)$ for stereotypically male/female job titles. Figure \ref{fig:eff_size} shows the distribution of $s(\mathbf{t},A,B)$ for DistilBERT, where stereotypical male/female targets are clearly distinct based on the sample bias. Figure \ref{fig:gpt_weat} shows the distribution of $s(\mathbf{t},A,B)$ for GPT, where stereotypical male and female terms are similarly distributed. First, we focus on the DistilBERT model (Figure \ref{fig:eff_size}), which clearly is biased with regard to the tested words. We compare two cases with different targets, such that the stereotypical target groups are better separable in one case (left plot), which one may describe as more severe or more obvious bias compared to the second case, where the target groups are almost separable (right plot). However, the effect size behaves contrary to this. \\
Despite this, when comparing Figures \ref{fig:eff_size} and \ref{fig:gpt_weat} , one can assume that large differences in effect sizes still reveal significant differences in social bias. While high effect sizes are reported in cases where both stereotypical groups are (almost) separable, we report low effect sizes when groups achieve similar individual biases. Furthermore, one should always report the p-value jointly with the effect size to get an impression on its significance. In other terms, WEAT is useful for qualitative bias analysis (or confirming biases), but not quantitatively.

\subsection{WEAT's individual bias}
\begin{table}[b]
\caption{Mean attribute difference $||\mathbf{\hat{a}} - \mathbf{\hat{b}}||$ for different language models given $25$ attribute pairs for gender.}
\label{tab:attr_diff}
\centering
\begin{tabular}{cc}
\toprule
Model Name & Mean Attribute Diff \\
\midrule
openai-gpt & 0.728 \\
gpt2 & 0.842 \\
bert-large-uncased & 1.123 \\
bert-base-uncased & 0.568 \\
distilbert-base-uncased  & 0.433 \\
roberta-base & 0.235 \\
distilroberta-base & 0.198 \\
electra-base-generator & 0.518 \\
albert-base-v2 & 1.123 \\
xlnet-base-cased & 5.206 \\
\bottomrule
\end{tabular}
\end{table}

In Theorem \ref{theor:weat_s_extrema} we discussed that WEAT's individual bias depends on the mean difference of attributes $||\mathbf{\hat{a}} - \mathbf{\hat{b}}||$. As shown in Table \ref{tab:attr_diff} these vary a lot between different language models. We report the two most extreme values of $0.198$ for \textit{distilroberta-base} and $5.206$ for \textit{xlnet-base-uncased}. With such differences we cannot compare sample biases based on their magnitude between different models.

\subsection{Direct Bias}
\begin{figure*}[tb]
\centering
\includegraphics[width=0.4\linewidth]{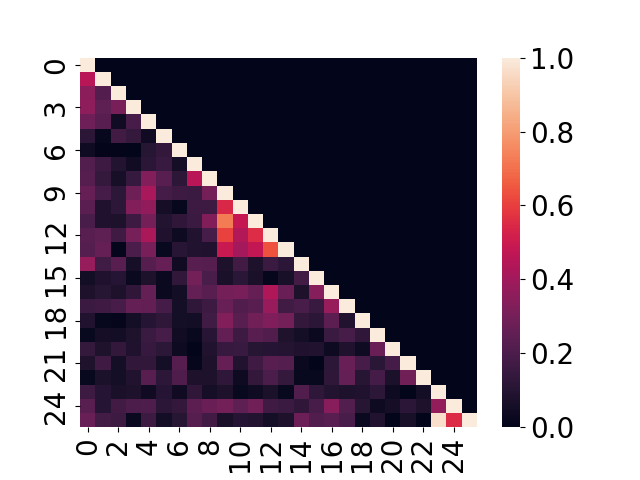}
\includegraphics[width=0.4\linewidth]{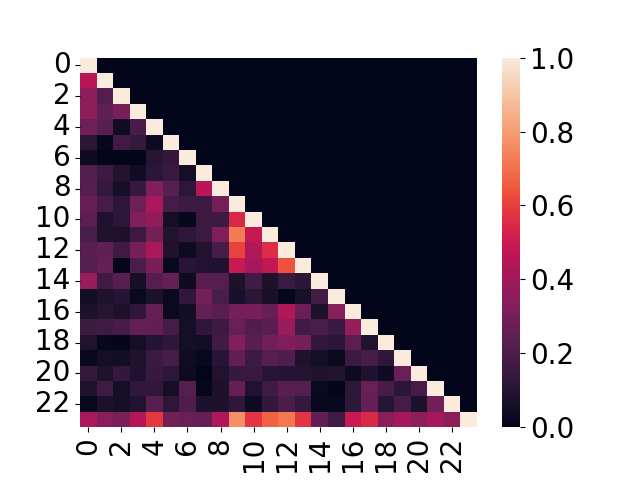}
\caption{Correlation bias directions of individual word pairs (left: 0-24, right: 0-22) and the first principal component (last row) as selected for the Direct Bias. The lowest row in the heatmap shows the correlation of individual bias directions with the first principal component.}
\label{fig:bias_dir}
\end{figure*}

Figure \ref{fig:bias_dir} shows the correlation of different bias directions on \textit{bert-base-uncased} embeddings. We report bias directions of individual word pairs such as $(\mathbf{man}, \mathbf{woman})$ (left plot: 0-24, right plot: 0-22) and the resulting bias direction as obtained by PCA (last row). Overall we report very low correlations between the individual bias directions. The first principal component reflects mostly individual bias direction 23 and 24 (left plot), which differ a lot from all other bias directions. On contrary, if we excluded word pairs 23 and 24 from the PCA (right plot), the first principal component would give a better estimate of bias directions 0-22. This shows that only one or few "outlier" pairs are sufficient to make the Direct Bias measure "bias" in a completely different way. From a practical point of view, by analyzing the correlation between individual bias directions we can get a good estimate whether the first principal component is a good estimate. Moreover, if we observe only weak correlations between bias directions from the selected word pairs, that is an indication that a 1-dimensional bias direction may not be sufficient to capture the relationship of sensitive groups with regard to which we want to measure bias. While Bolukbasi et al. \cite{bolukbasi} did not explicitly define that case for the Direct Bias, they proposed to use a bias subspace, defined by the $k$ first principal components, for their Debiasing algorithm, which is related to the Direct Bias. Accordingly, this could be applied to the Direct Bias. Apart from that, one should verify how well the bias direction or bias subspace obtained by PCA represents individual bias directions to make sure that they're actually measuring bias in the assumed way.

\section{Conclusion}
\label{sec:conclusion}
In this work, we introduce formal properties for cosine based bias scores, concerning their meaningfulness for quantification of social bias. We show that WEAT and the Direct Bias have  theoretical flaws that limits their ability to quantify bias. Furthermore, we show that these issues have a real impact when applying these bias scores on state-of-the-art language models. These findings should be considered in the experimental design when evaluating social bias with one of these measures.
Future works could build on the proposed properties to analyze other scores from the literature or propose a score that is better suited for bias quantification.
The findings of our theoretical analysis open the question, whether the limitations of cosine based scores reported in the literature are due to the theoretical flaws of distinct scores, which are highlighted by our analysis, rather than limitations of geometrical properties as a sign of bias. This is an important question that should be addressed in future work.
In general, we encourage other researchers to take an effort to bring the various bias measures from the literature into context and to highlight their properties and limitations, which is critical to derive best practices for bias detection and quantification.

\section*{Acknowledgments}
Funded by the Ministry of Culture and Science of North-Rhine-Westphalia in the frame of the project SAIL, NW21-059A.

\bibliographystyle{unsrt}  
\bibliography{references}

\newpage
\section*{\uppercase{Appendix}}
In order to show that the effect size of WEAT is \extremaVgl \;(see Theorem \ref{theor:weat_extrema}), we need the following lemma.

\begin{lemma}\label{theor:schrankeAllg}
Let $x_1,...,x_n \in \mathbb{R}$ be real numbers. Let $\hat{\mu},\hat{\sigma}$ denote the empirical estimate of mean and standard deviation of the $x_i$. Then, for any selection of indices $i_1,...,i_m$, with $i_j \neq i_k$ for $j \neq k$, the following bound holds
\begin{align*}
    \left| \sum_{j = 1}^m \frac{x_{i_j} - \hat{\mu}}{\hat{\sigma}} \right| \leq \sqrt{m \cdot (n-m)}.
\end{align*}
Furthermore, for $0 < m < n$ the bound is obtained if and only if all selected resp. non-selected $x_i$ have the same value, i.e. $x_{i_j} = \hat{\mu} + s\sqrt{\frac{n-m}{m}}\hat{\sigma}$ and all other $x_k = \hat{\mu} -s\sqrt{\frac{m}{n-m}}\hat{\sigma}$ with $s \in \{-1,1\}$.
\end{lemma}

\begin{proof}
\label{proof:schrankeAllg}
For cases $m = 0$ or $m = n$ the statement is trivial. So assume $0 < m < n$.
\newcommand{\x}{\mathbf{x}}
\newcommand{\s}{\mathbf{s}}
Let $f(x) = ax+b$ be an affine function. Then the images of $x_i$ under $f$ have mean $a\hat{\mu}+b$ and standard deviation $|a|\hat{\sigma}$. On the other hand, we have
\begin{align*}
    \frac{f(x_i)-(a\hat{\mu}+b)}{|a|\hat{\sigma}}
    &= \frac{(ax_i+b)-(a\hat{\mu}+b)}{|a|\hat{\sigma}}
    = \text{sgn}(a)\frac{x_i-\hat{\mu}}{\hat{\sigma}}.
\end{align*}
Thus, applying $f$ does not change the bound and therefore we may reduce to case of $\hat{\mu} = 0$ and $\hat{\sigma} = 1$. This allows us to rephrase the problem of finding the maximal bound as an quadratic optimization problem:
\begin{align*}
    \min \quad& \s^\top \x \\
    \text{s.t.} \quad & \x^\top \x = n\\
                      & \mathbf{1}^\top \x = 0,
\end{align*}
where $\s = (1,...,1,0,...,0)^\top$, $\x = (x_1,...,x_n)^\top$ and $\mathbf{1}$ denotes the vector consisting of ones only. Notice, that we assumed w.l.o.g. that $i_1,...,i_m = 1,...,m$. Furthermore, we made use of the symmetry properties to replace $\max |\s^\top \x|$ by the minimizing statement above, $\hat{\mu} = 0$ is expressed by the last and $\hat{\sigma} = 1$ by the first constrained (recall that $\hat{\sigma} =  \sqrt{1/n \x^\top \x-\hat{\mu}^2}$).
Notice, that $\nabla_x \x^\top\x-n = 2\x$ and $\nabla_x \mathbf{1}^\top \x = \mathbf{1}$ are linear dependent if and only if $\x = a \mathbf{1}$ for some $a \in \mathbb{R}$, thus, as $0 = a \mathbf{1}^\top \mathbf{1} = an$ if and only if $a = 0$ and $(0\mathbf{1})^\top (0\mathbf{1}) = 0$, there is no feasible $\x$ for which the KKT-conditions do not hold and we may therefore use them to determine all the optimal points.

The Lagrangien of the problem above and its first two derivatives are given by
\begin{align*}
    L(\x,\lambda_1,\lambda_2) &= \s^\top \x - \lambda_1 (\x^\top I \x - n) - \lambda_2 \mathbf{1}^\top \x \\
    \nabla_x L(\x,\lambda_1,\lambda_2) &= \s - 2 \lambda_1 \x - \lambda_2 \mathbf{1} \\
    \nabla^2_{x,x} L(\x,\lambda_1,\lambda_2) &= - 2 \lambda_1 I.
\end{align*}
We can write $\nabla_x L(\x,\lambda_1,\lambda_2) = 0$ as the following linear equation system:
\begin{align*}
    \begin{bmatrix} 2x_1 & 1 \\ 2x_2 & 1 \\ \vdots & \vdots \\ 2x_n & 1 \end{bmatrix} \begin{bmatrix} \lambda_1 \\ \lambda_2 \end{bmatrix} &= \underbrace{\begin{bmatrix} 1 \\ 1 \\ \vdots \\ 0 \end{bmatrix}}_{=\s}.
\end{align*}
Subtracting the first row from row $2,...,m$ and row $m+1$ from row $m+2,...,n$ we see that $2(x_k-x_1)\lambda_1 = 0$ for $k=1,...,m$ and $2(x_k-x_{m+1})\lambda_1 = 0$ for $k = m+2,...,n$, which either implies $\lambda_1 = 0$ or $x_1=x_2 = ... = x_m$ and $x_{m+1}=x_{m+2} = ... = x_n$. However, assuming $\lambda_1 = 0$ would imply that $\lambda_2 = 1$ from the first row and $\lambda_2 = 0$ from the $m+1$th row, which is a contradiction. Thus, we have  $x_1=x_2 = ... = x_m$ and $x_{m+1}=x_{m+2} = ... = x_n$. But the second constraint from the optimization problem can then only be fulfilled if $mx_1+(n-m)x_{m+1} = 0$ and this implies $x_{m+1} = -\frac{m}{n-m}x_1$. In this case the first constraint is equal to $n = m x_1^2 + (n-m) \left(\frac{m}{n-m} x_1\right)^2$
, which has the solution $x_1 = \pm\sqrt{\frac{n-m}{m}}$.

Set $\x^* = (-\sqrt{\frac{n-m}{m}},...,-\sqrt{\frac{n-m}{m}},\sqrt{\frac{m}{n-m}},...,\sqrt{\frac{m}{n-m}})$. Then $\x^*$ and $-\x^*$ are the only possible KKT points as we have just seen. Plugging $\x^*$ into the equation system above and solving for $\lambda_{1/2}$ we obtain
\begin{align*}
    \lambda_1^* &= -\frac{1}{2 \left(\sqrt{\frac{m}{n - m}} + \sqrt{\frac{n - m}{m}}\right)}, \; \; \;
    \lambda_2^* = \frac{\sqrt{\frac{m}{n - m}}}{\sqrt{\frac{m}{n - m}} + \sqrt{\frac{n-m}{m}}}
\end{align*}
Now, as $\nabla^2_{x,x} L(\x^*,\lambda_1^*,\lambda_2^*) = \left(\sqrt{\frac{m}{n - m}} + \sqrt{\frac{n - m}{m}}\right)^{-1} I$ is positive definite, we see that $\x^*$ is a global optimum, indeed. The statement follows.
\end{proof}

\end{document}